\documentclass[]{article}

\usepackage[final,nonatbib]{nips_2017}

\usepackage[utf8]{inputenc} 
\usepackage[T1]{fontenc}    
\usepackage{hyperref}       
\usepackage{url}            
\usepackage{booktabs}       
\usepackage{amsfonts}       
\usepackage{nicefrac}       
\usepackage{microtype}      
\usepackage{authblk}
\usepackage{bm}
\usepackage{enumitem}

\makeatletter
\renewcommand\AB@affilsepx{,\quad \protect\Affilfont}

\usepackage{tabularx}
\usepackage{graphicx}
\usepackage{multirow}
\usepackage{amsmath,amsfonts,amssymb, amsthm}
\usepackage{mathtools}
\usepackage{subcaption}
\usepackage{physics}
\usepackage{multirow, makecell}
\usepackage{hyperref}
\usepackage{fixltx2e}
\usepackage[leftmargin=2em]{quoting}
\quotingsetup{vskip=0pt}
\usepackage{enumitem}

\usepackage{geometry}
\newtheorem{theorem}{Theorem}

\newtheorem{lemma}[theorem]{Lemma}

\newenvironment{proof-sketch}{\noindent{\bf Sketch of Proof}\hspace*{1em}}{\qed\bigskip}
\newenvironment{proof-idea}{\noindent{\bf Proof Idea}\hspace*{1em}}{\qed\bigskip}
\newenvironment{proof-of-lemma}[1]{\noindent{\bf Proof of Lemma #1}\hspace*{1em}}{\qed\bigskip}
\newenvironment{proof-attempt}{\noindent{\bf Proof Attempt}\hspace*{1em}}{\qed\bigskip}

\newcommand{\Exp}{\mathbb{E}}
\newcommand{\Reals}{\mathbb{R}}

{\makeatletter
 \gdef\xxxmark{%
   \expandafter\ifx\csname @mpargs\endcsname\relax 
     \expandafter\ifx\csname @captype\endcsname\relax 
       \marginpar{\textcolor{red}{xxx~}}
     \else
       \textcolor{red}{xxx~}
     \fi
   \else
     \textcolor{red}{xxx~}
   \fi}
 \gdef\xxx{\@ifnextchar[\xxx@lab\xxx@nolab}
 \long\gdef\xxx@lab[#1]#2{{\bf [\xxxmark \textcolor{red}{#2} ---{\sc #1}]}}
 \long\gdef\xxx@nolab#1{{\bf [\xxxmark \textcolor{red}{#1}]}}
}

\DeclareMathOperator*{\argmax}{arg\,max}

\newcommand{\leg}{leg}
\newcommand{\adv}{adv}

\newcommand{\dir}[1]{\ensuremath{{\vec{d}}_{\text{#1}}}}
\newcommand{\mindist}[1]{\ensuremath{\text{MIN-DIST}_{#1}}}
\newcommand{\interdist}[1]{\ensuremath{\text{INTER-DIST}_{{#1}}}}

\newcommand{\inpspace}{D}
\newcommand{\Dist}{\mu}
\newcommand{\Dm}{\ensuremath{\Dist_{-1}}}
\newcommand{\Dp}{\ensuremath{\Dist_{+1}}}

\newcommand{\meanm}{\ensuremath{\Exp_{\Dm}[\vec{x}]}}
\newcommand{\meanp}{\ensuremath{\Exp_{\Dp}[\vec{x}]}}

\newcommand{\meanfm}{\ensuremath{\Exp_{\Dm}{[\phi(\vec{x})]}}}
\newcommand{\meanfp}{\ensuremath{\Exp_{\Dp}{[\phi(\vec{x})]}}}
\renewcommand{\vec}[1]{{\bm {#1}}}

\title{The Space of Transferable Adversarial Examples}
\author[1]{Florian Tramèr}
\author[2]{Nicolas Papernot}
\author[3]{Ian Goodfellow}
\author[1]{Dan Boneh} 
\author[2]{Patrick McDaniel}
\affil[1]{Stanford University}
\affil[2]{Pennsylvania State University}
\affil[3]{Google Brain}
\date{}

\begin{document}

\maketitle

\begin{abstract}
Adversarial examples are maliciously perturbed inputs 
designed to mislead machine learning
(ML) models at test-time.
They often \emph{transfer}: 
the same adversarial example fools more than one model.

In this work, we propose novel methods for estimating the 
previously unknown \emph{dimensionality} of the space 
of adversarial inputs. 
We find that adversarial examples 
span a contiguous subspace of large (\textasciitilde 25) dimensionality.
Adversarial subspaces with higher dimensionality are more likely
to intersect.
We find that for two different models, a
significant fraction of their subspaces is shared, 
thus enabling transferability.

In the first quantitative analysis of the similarity of 
different models' decision boundaries, we show that these boundaries 
are actually close in \emph{arbitrary} directions, whether adversarial 
or benign.
We conclude by formally studying the \emph{limits}
of transferability. We derive (1) sufficient conditions
on the \emph{data distribution} that imply transferability for simple model
classes and (2) examples of scenarios in which transfer does not occur.
These findings indicate that it may be possible to design defenses
against transfer-based attacks, even for models that are vulnerable
to direct attacks.

\end{abstract}
\section{Introduction}
\label{sec:introduction}

Through slight perturbations of a machine learning (ML) model's inputs at test time, 
it is possible to generate \emph{adversarial examples} that cause the model to
misclassify at a high rate~\cite{biggio2013evasion,szegedy2013intriguing}.
Adversarial examples can be used to craft human-recognizable images
that are misclassified by computer vision
models~\cite{szegedy2013intriguing, 
	goodfellow2014explaining,
	moosavi2016deepfool,
	kurakin2016adversarial, 
	liu2016delving},
software containing malware but
classified as benign~\cite{laskov2014practical,xu2016automatically,grosse2016adversarial,hu2017generating}, and game
environments that force reinforcement learning agents to
misbehave~\cite{huang2017adversarial, behzadan2017vulnerability, lin2017tactics}.

Adversarial examples often \emph{transfer} across models 
~\cite{szegedy2013intriguing,
	goodfellow2014explaining,papernot2016transferability}: inputs 
generated to evade a specific model also mislead other models trained 
for the same task.
Transferability is an obstacle to
 secure deployment of ML models as it enables 
simple \emph{black-box attacks} against ML systems. 
An adversary can train a local
model---possibly by issuing prediction queries to the targeted 
model~\cite{papernot2016practical}---and use it to craft 
adversarial examples that transfer to the target model~\cite{szegedy2013intriguing}.
To defend against such attacks, it is necessary to have
a better understanding of why adversarial examples transfer.

\vspace*{-0.12in}

\paragraph{Adversarial Subspaces. }
Empirical evidence has shown that, rather than being scattered randomly
in small pockets, adversarial examples occur in large, contiguous
regions~\cite{goodfellow2014explaining,WardeFarley16}.
The dimensionality of these subspaces is relevant to the transferability
problem: the higher the dimensionality, the more likely it is
that the subspaces of two models will intersect significantly. 

\vspace*{-0.05in}

In this work, we thus directly estimate the dimensionality of these 
subspaces. 
We introduce methods for finding 
multiple \emph{orthogonal} adversarial directions and show 
that these perturbations span a multi-dimensional contiguous space of 
misclassified points.
We measure transferability of these subspaces 
on datasets for which diverse model classes 
attain high accuracy: 
digit classification (MNIST)~\cite{lecun1998gradient}
and malware detection (DREBIN)~\cite{arp2014drebin}.\footnote{
	We use a balanced subset of the DREBIN dataset, 
	pre-processed using feature selection to have
	$1{,}000$ features (in lieu of about $500,000$). 
	The models we train get over $94\%$ accuracy on this dataset.}
For example, we find that adversarial examples that transfer between two fully-connected 
networks trained on MNIST form a 25-dimensional space. 

In addition to sharing many dimensions of the adversarial subspace, we
find empirically that the boundaries of this space lie at similar distances
from legitimate data points in adversarial directions (e.g., indicated by an
adversarial example).
More surprisingly, models from different hypothesis 
classes learn boundaries that are close in \emph{arbitrary} 
directions, whether adversarial or benign (e.g., the direction defined
by two legitimate points in different classes).
We find that when moving into any direction away 
from data points, the distance to the  
model's decision boundary is on average larger than the distance 
\emph{separating} the boundaries of two models in that 
direction. Thus, adversarial perturbations 
that send data points sufficiently over a model's decision boundary 
likely transfer to other models.

\paragraph{The Limits of Transferability. } Given the empirical 
prevalence of transferability, it is natural to ask whether this 
phenomenon can be explained by simple properties of datasets, model classes, or
training algorithms.
In Section~\ref{sec:limits}, we consider the following informal hypothesis:
\vspace{-0.1em}
\emph{
\begin{quoting}
If two models achieve low error for some task while also exhibiting 
low robustness to adversarial examples, adversarial examples crafted 
on one model transfer to the other.
\end{quoting} 
}
\vspace{-0.1em}

This hypothesis is pessimistic: it implies that a model cannot be 
secure against adversarial examples transferred from other models 
(i.e, black-box attacks) unless it is also robust to adversarial 
examples crafted with knowledge of the model's parameters 
(i.e., white-box attacks). While this hypothesis holds in certain 
contexts, we show that it is \emph{not} true in the general case. 

We derive sufficient conditions on the data distribution 
that imply a form of the above hypothesis for a set of simple model classes. Namely, we 
prove transferability of \emph{model-agnostic} perturbations, 
obtained by shifting data points in
the direction given by the difference in class means (in input space).
These adversarial examples are by design effective against 
linear models. Yet, we 
show, both formally and empirically, that they can transfer 
to higher-order (e.g., quadratic) models.

However, we exhibit a counter-example to the above hypothesis by 
building a variant of the MNIST dataset for which adversarial 
examples fail 
to transfer between linear and quadratic models. 
Our experiment
suggests that transferability is not an inherent 
property of non-robust ML models. 

\paragraph{Our Contributions. }
To summarize, we make the following contributions:
\begin{itemize}[topsep=0pt, itemsep=0pt,leftmargin=0.25in]
	\itemsep0em

\item We introduce methods for finding multiple independent attack 
directions, enabling us to directly measure the dimensionality of the
adversarial subspace for the first time.
\item We perform the first quantitative study of the similarity of models' decision 
boundaries and show that models from different hypothesis classes 
learn decision boundaries that lie very close to one-another in both 
adversarial and benign directions.
\item In a formal study, we identify sufficient 
conditions for  transferability of model-agnostic perturbations, as well 
as tasks where  adversarial example transferability fails to hold.
\end{itemize}
\vspace*{-0.05in}

\section{Adversarial Example Generation}
\label{sec:background}

This work considers attacks mounted with adversarial
examples~\cite{szegedy2013intriguing,biggio2013evasion} to mislead 
ML models at test time.
We focus on techniques 
that aim to fool a model into producing \emph{untargeted} 
misclassifications (i.e., the model predicts any 
class other than the ground truth). We first introduce some 
notation.

\vspace*{-0.05in}
\def\arraystretch{1.2}
\begin{tabularx}{\textwidth}{r  X}
$\vec{x}, y$ & A clean input from some domain $D$ and its corresponding label. \\
$\vec{x}^*$ & An adversarial input in $D$. To enforce $\vec{x}^* \in D$, 
we either clip all pixels to the range $[0,1]$ for MNIST, or we round 
all feature values to a binary value for DREBIN.\\
$\vec{r}, \epsilon$ & The perturbation vector added to an input to 
create an adversarial example: $\vec{x}^* = \vec{x} + \vec{r}$,  
and its magnitude i.e., $\epsilon = \norm{\vec{r}}$ for some 
appropriate norm (e.g., $\ell_\infty$ or $\ell_2$). \\
$J(\vec{x}, y)$ & The loss function used to train model $f$ (e.g., 
cross-entropy).
\end{tabularx}
\def\arraystretch{1.0}

\paragraph{Fast Gradient [Sign] Method (FG[S]M).} 
For most of the experiments in this paper, we use
the method proposed in~\cite{goodfellow2014explaining}. 
The FGSM is an efficient technique for 
generating adversarial examples with a fixed $\ell_\infty$-norm: 
$
\vec{x}^* = \vec{x} + \varepsilon \cdot \mathtt{sign}\left(\nabla_{\vec{x}} J(\vec{x}, y)\right)
$.
We will consider variants of this approach that
constrain the perturbation using
other $\ell_p$ norms:
\begin{equation}
\label{eq:fgm}
\vec{x}^* = \vec{x} + \varepsilon \cdot \nabla_{\vec{x}} J(\vec{x}, y) / \|\nabla_{\vec{x}} J(\vec{x}, y)\|
\end{equation}
For general $\ell_p$ norms, we drop the ``sign'' in the acronym and 
simply call it the \emph{fast gradient method}.
We use the implementations provided by 
the \texttt{cleverhans~1.0} library~\cite{papernot2016cleverhans}.
Our most successful techniques for finding multiple adversarial 
directions, described in Section~\ref{sec:exploration}, are based on 
the FGM. We discuss other attacks we considered in 
Appendix~\ref{apx:multi-adv}.

\section{Exploring the Space of Transferable Adversarial Examples}
\label{sec:exploration}

Our aim is to evaluate the dimensionality of the adversarial
subspace that transfers between models. 
The FGM discussed above computes
adversarial examples by taking one step in the \emph{unique} 
optimal direction under a first-order approximation of the model's 
loss.
However, we know that adversarial examples form a dense space that is 
(1) at least two-dimensional~\cite{WardeFarley16}, and (2) occupies a 
negligible fraction of the entire input 
space~\cite{szegedy2013intriguing}.
This raises the question of how ``large'' this space is. 
We considered several techniques to find multiple 
\emph{orthogonal} adversarial
directions for an input point. 
Here, we only describe the most successful one. Others are described 
in Appendix~\ref{apx:multi-adv}.

\subsection{Gradient Aligned Adversarial Subspace (GAAS).}

This technique directly estimates the dimensionality of the 
adversarial subspace under a first-order approximation of the loss 
function.
For an input $\vec{x}$, we search the space of perturbations $\vec{r}$ (with
$\norm{\vec{r}}_2 \leq \epsilon$) that result in a significant increase in 
loss, i.e.,
$J(\vec{x}+\vec{r}, y) \geq J(\vec{x}, y) + \gamma$, for a $\gamma > 0$.
We recast this as the problem
of finding a maximal set of orthogonal perturbations 
$\vec{r}_1, \vec{r}_2, \dots, \vec{r}_k$, 
satisfying $\norm{\vec{r}_i}_2 \leq \epsilon$ and 
$\vec{r}_i^\top \nabla_{\vec{x}} J(\vec{x}, y) \geq \gamma$.
We will need the following simple result:

\begin{lemma}
	\label{lemma:orthogonal}
	Given $\vec{g} \in \Reals^d$ and $\alpha \in [0, 1]$. The maximum 
	number $k$ of orthogonal vectors $\vec{r}_1, \vec{r}_2, \dots \vec{r}_k \in \Reals^d$ 
	satisfying $\|\vec{r}_i\|_2 \leq 1$ and $\vec{g}^\top \vec{r}_i \geq \alpha \cdot \|\vec{g}\|_2$ 
	is  $k = \min\left\{\left\lfloor \frac{1}{\alpha^2}\right\rfloor, d\right\}$.
\end{lemma}

The proof is in Appendix~\ref{apx:proof-ortho}. Applying
Lemma~\ref{lemma:orthogonal} with $\alpha = \gamma^{-1}\epsilon
	\norm{\nabla_{\vec{x}} J(\vec{x}, y)}_2$,
and scaling the obtained unit
vectors by $\epsilon$, yields a set of orthogonal perturbations that
satisfy $\norm{\vec{r}_i}_2 \leq \epsilon$ and $\vec{r}_i^\top \nabla_{\vec{x}} J(\vec{x}, y) \geq
\gamma$, as desired. 
Informally, the number of orthogonal
adversarial directions is proportional to the increase in loss $\gamma$ 
(a proxy for the distance from $\vec{x}$ to the decision boundary) and 
inversely proportional to the \emph{smoothness} of the loss function 
and the perturbation magnitude $\epsilon$.
If the value of
$\gamma$ that results in a misclassification is unknown, 
we try the attack with multiple values of $\gamma$, and retain the 
value that results in a set with the most misclassified perturbations. 
An illustration of the GAAS method is given in Figure~\ref{fig:gaas}.
Finding an analog maximal construction of orthogonal perturbations
for other norms (e.g., $\ell_\infty$) is a nice open problem.

\subsection{Experiments}
\label{ssec:additional-dirs-eval}

\begin{figure}
	\centering
	
	\begin{minipage}[t]{0.48\textwidth}
		\centering
		\includegraphics[width=0.9\textwidth]{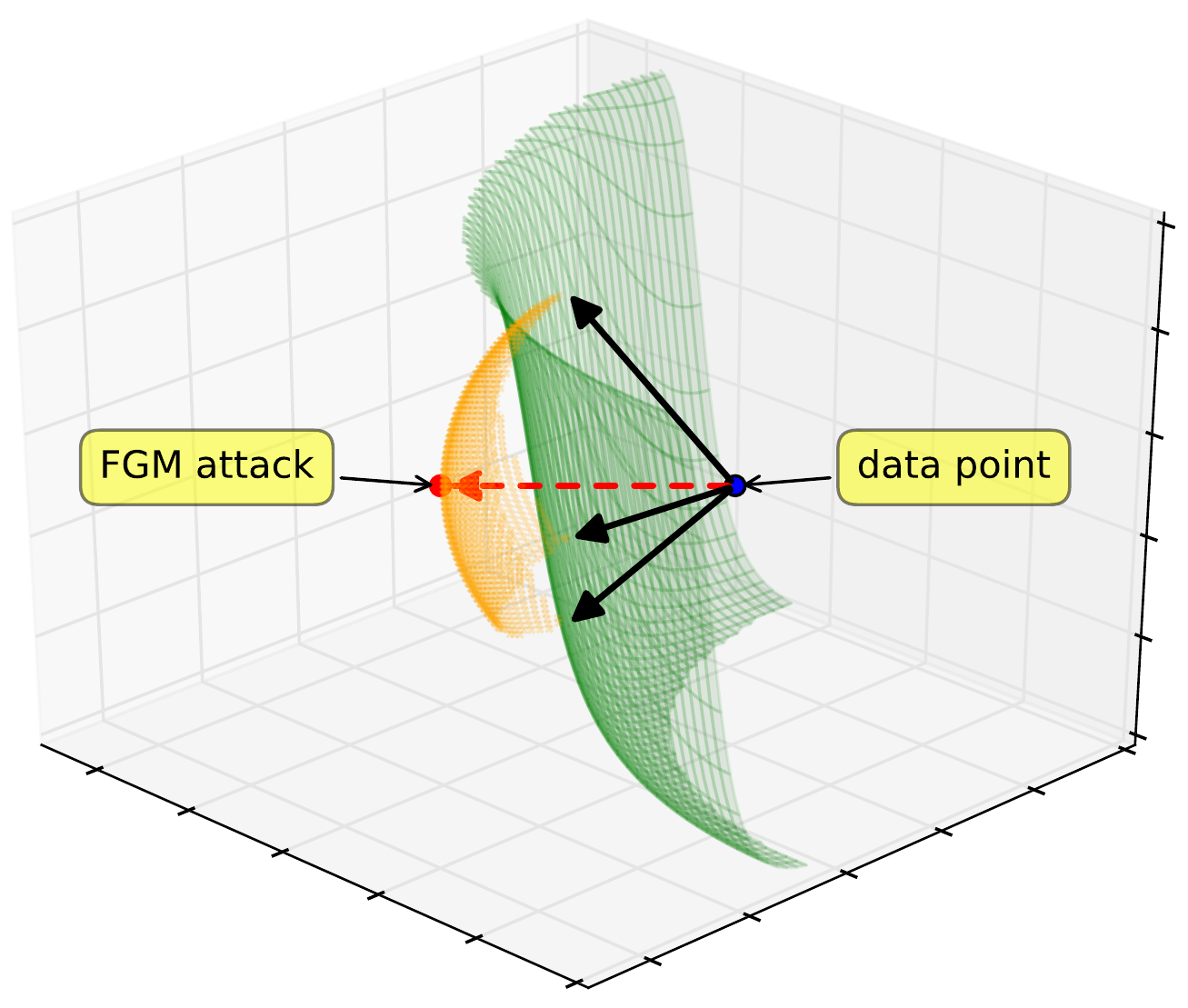}
		\caption{Illustration of the Gradient Aligned Adversarial Subspace (GAAS). 
			The gradient aligned attack (red arrow) crosses the decision boundary. 
			The black arrows are orthogonal vectors aligned with the gradient 
			that span a subspace of potential adversarial inputs (orange).}
		\label{fig:gaas}
	\end{minipage}
	~
	\hfill
	~
	\begin{minipage}[t]{0.48\textwidth}
		\includegraphics[width=\textwidth]{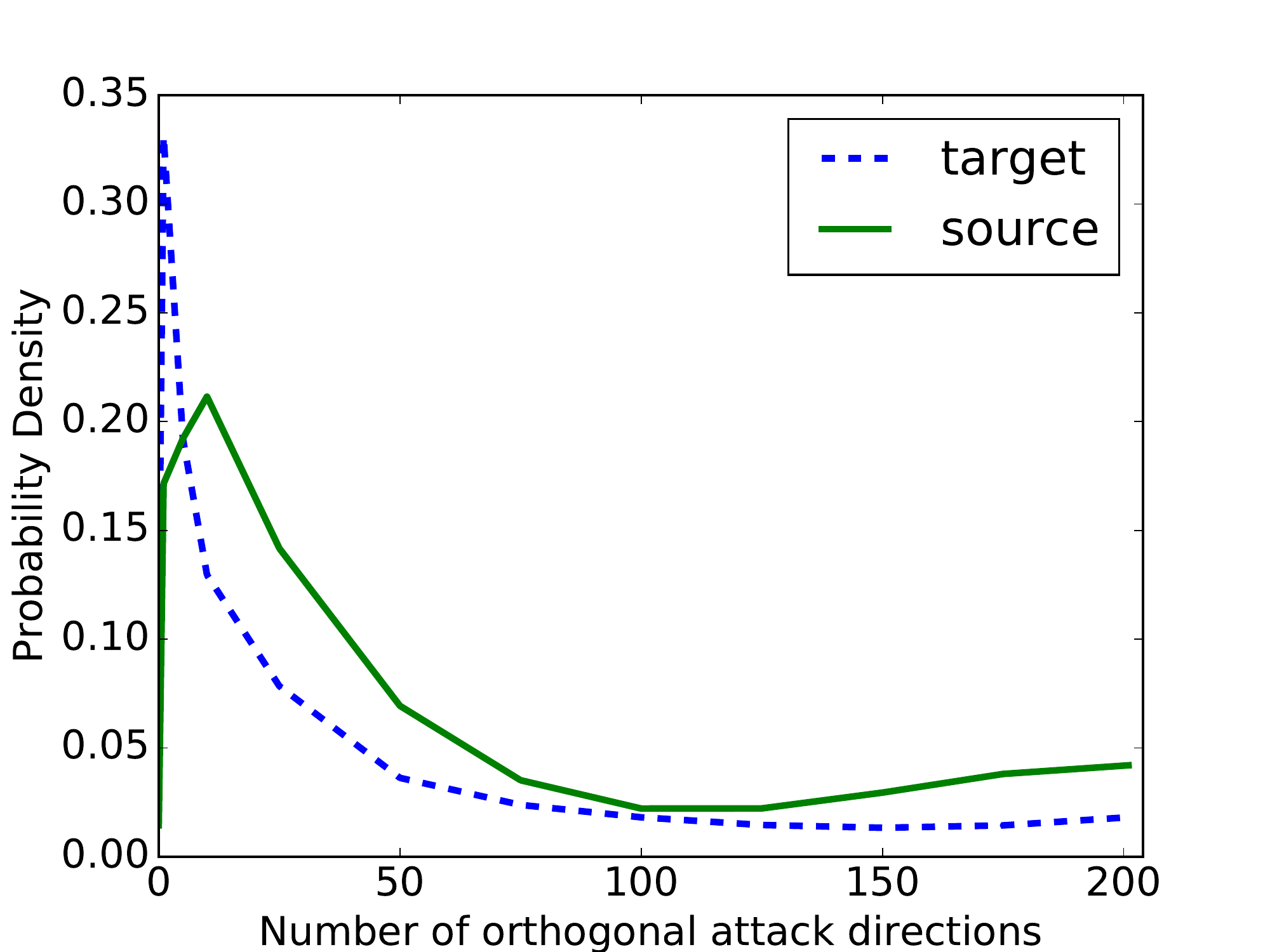}
		\caption{Probability density function of the number of 
			successful orthogonal adversarial perturbations 
			found by the GAAS method on the source DNN model, 
			and of the number of perturbations 
			that transfer to the target DNN model.} 
		\label{fig:cdf-directions}
	\end{minipage}
	
\end{figure}

The GAAS method was the most successful at
finding a large number of orthogonal attack directions 
(see Table~\ref{tab:multi-adv} in Appendix~\ref{apx:multi-adv} for 
results with other techniques). This gives further evidence that 
neural networks generalize in an overly linear fashion to 
out-of-sample data, as first argued 
by~\cite{goodfellow2014explaining}.

We first use two fully connected networks trained on MNIST
(results on DREBIN are in Appendix~\ref{apx:multi-adv}).
The source model f$_{\text{src}}$ (a two layer variant of 
architecture C from Table~\ref{table:mnist_archis}) for crafting 
adversarial examples is shallower than the target model 
f$_{\text{target}}$ (architecture C). 
We compute perturbations $\vec{r}$ with norm $\norm{\vec{r}}_2 \leq 5$.
Figure~\ref{fig:cdf-directions} plots the probability 
density function of the dimensionality of the adversarial 
subspace for each input.
On average, we find $44.28$ orthogonal perturbations (and over 
$200$ for the most vulnerable inputs) and $24.87$ directions that 
transfer to model f$_{\text{target}}$.
These perturbations span a dense subspace of adversarial
inputs: by randomly sampling a perturbation in the spanned space 
(with an $\ell_2$-norm of $5.0$), we fool model f$_{\text{src}}$ in 
$99\%$ of cases, and model
f$_{\text{target}}$ in $89\%$ of cases.

We repeat the experiment for two CNNs (model B in
Table~\ref{table:mnist_archis} as source, and A as target). 
The transfer rate for the FGM is lower for these 
models ($68\%$). 
We find fewer orthogonal perturbations on the source model ($15.18$), 
and, expectedly, fewer that transfer ($2.24$). Yet, 
with high probability,
points in the spanned adversarial subspace mislead
the source ($80\%$) and target ($63\%$) models.

\section{Decision Boundary Similarity enables Transferability}
\label{sec:boundaries}

Evidenced by Section~\ref{sec:exploration}, the existence of 
large transferable adversarial subspaces suggests 
that the decision boundaries learned by the source and target 
models must be extremely close to one another.
In this section, we quantitatively characterize
this similarity in both adversarial and 
benign directions. 
Our measurements show that when moving away from 
data points in different directions, the distance between two 
models' decision boundaries is smaller than the distance 
separating the data points from either boundary.
We further find that adversarial training~\cite{szegedy2013intriguing,goodfellow2014explaining} 
does not significantly ``displace'' the learned decision boundary,
thus leaving defended models vulnerable to black-box attacks.

\subsection{Distance Definitions}

Evaluating the decision 
boundary of an ML model in high-dimensions is usually intractable.
We thus restrict ourselves to measuring similarity in three linear 
directions, which are representative of the
model behavior on and off the data manifold. 
These are illustrated 
in Figure~\ref{fig:distance-directions}. 
Each direction is defined by a unit-norm vector 
relative to a data point $\vec{x}$ and some other point $\vec{x}'$, as 
$\dir{}(f, \vec{x}) \coloneqq 
\frac{\vec{x}' - \vec{x}}{\norm{\vec{x}' -\vec{x}}_2}$.

Distance is measured using the 
$\ell_2$ norm (we present similar results with the $\ell_1$ norm
in Appendix~\ref{ap:dir-dis}).
The point $\vec{x}'$ is defined differently for each direction, as follows:
\begin{enumerate}[itemsep=3pt,topsep=0pt]
	\item \textbf{Legitimate direction---}Written $\dir{\leg}(f, \vec{x})$, 
	it is defined by $\vec{x}$ and the closest data point $\vec{x}'$ in the test set
	with a different class label than $\vec{x}$.
	\item \textbf{Adversarial direction---}Written
	$\dir{\adv}(f, \vec{x})$, it is defined by $\vec{x}$ and an adversarial example 
	$\vec{x}' \coloneqq adv(f, \vec{x})$ generated
	from $\vec{x}$ to be misclassified by $f$.
	\item \textbf{Random direction---}Written
	$\dir{rand}(f, \vec{x})$, it is defined by $\vec{x}$ and an input $\vec{x}'$ drawn 
	uniformly over the input domain, conditioned on $\vec{x}'$ being 
	classified by $f$ in a different class than $\vec{x}$.
\end{enumerate}

\begin{figure}[t]
	\centering
	\includegraphics[width=0.95\textwidth]{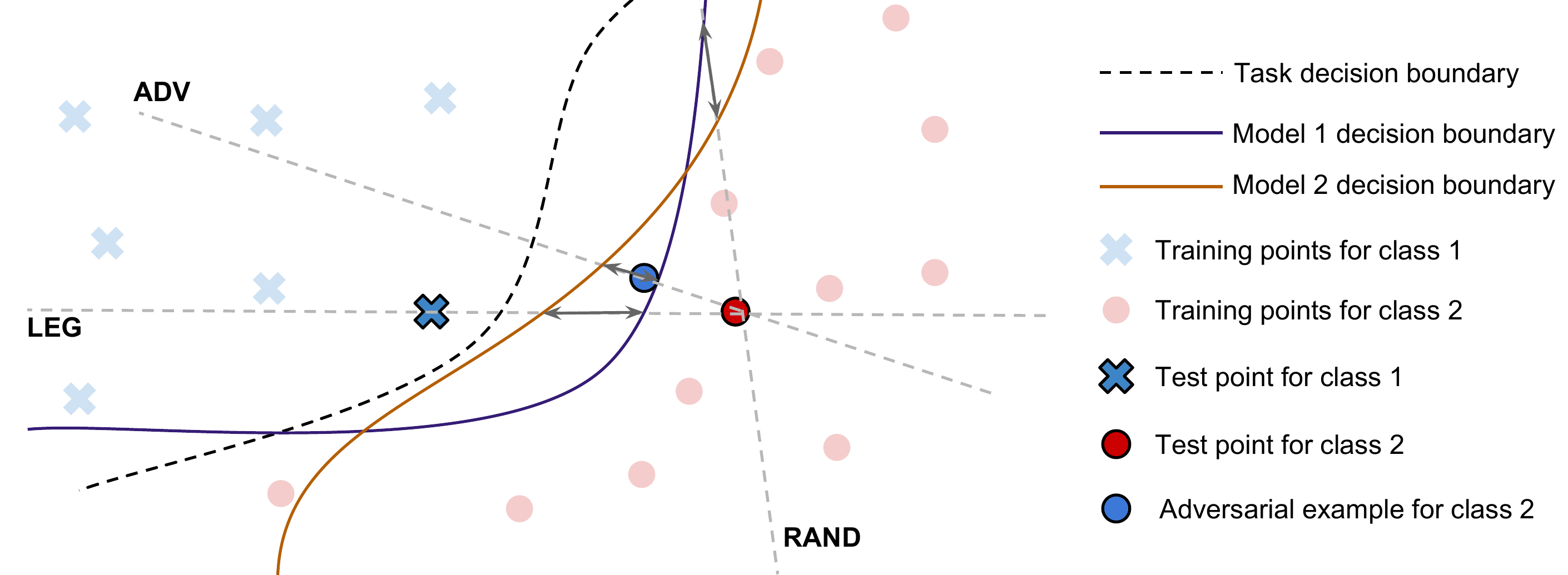}
	\caption{The three directions (Legitimate, Adversarial and Random) used throughout
		Section~\ref{sec:boundaries} to measure the distance
		between the decision boundaries of two models. The gray 
		double-ended arrows illustrate the \emph{inter-boundary}
		distance between the two models in each direction.\\[-2em]}
	\label{fig:distance-directions}
\end{figure}

These directions are used to define two types of metrics:

\textbf{Minimum Distances:}
Given one of the above directions $\dir{}$, the 
\emph{minimum distance} from a point $\vec{x}$ to the model's decision 
boundary is defined as:
\begin{equation}
\label{eq:min-dist}
\mindist{\dir{}}(f, \vec{x}) \coloneqq \arg\min_{\epsilon > 0} f(\vec{x}+\varepsilon\cdot\dir{}) \neq f(\vec{x}) \;.
\end{equation}
\textbf{Inter-Boundary Distances:} We are interested in the distance 
\emph{between the decision boundaries} of different models 
(i.e., the bi-directional arrows in Figure~\ref{fig:distance-directions}). 
For a point $\vec{x}$ and a direction $\dir{}$ computed 
according to $f_1$ (e.g., $\dir{}=\dir{\adv}(f_1, \vec{x})$), we define the 
\emph{inter-boundary distance} between $f_1$ and $f_2$ as:
\begin{equation}
\label{eq:inter-boundary-dist}
\interdist{\dir{}}(f_1, f_2, \vec{x}) \coloneqq |\mindist{\dir{}}(f_1, \vec{x}) - \mindist{\dir{}}(f_2, \vec{x})| \;.
\end{equation}
In the adversarial direction, the \emph{inter-boundary} distance
is directly related to adversarial example transferability. If the 
distance is small, adversarial examples crafted from $f_1$ that cross
the decision boundary of $f_1$ are likely to also cross the decision 
boundary of $f_2$: they \emph{transfer}.

\subsection{Experiments with the MNIST and DREBIN Datasets}

We study three models trained on the MNIST and DREBIN tasks: a logistic
regression (LR), support vector machine (SVM), and  neural 
network (DNN architecture C in Table~\ref{table:mnist_archis}).
For all pairs of models, we compute the mean inter-boundary
distance in the three directions: legitimate, adversarial and random.
The mean minimum distance in the same direction acts as a baseline for these
values. 

In the input domain, we measure distances with the $\ell_2$ norm
and report their mean value
over all points in the test set.
We use the ($\ell_2$ norm) FGM~\cite{goodfellow2014explaining} to find adversarial 
directions 
for differentiable models (the LR and DNN) and the method 
from~\cite{papernot2016transferability}
for the SVM. To compute the minimum 
distance from a point to a decision boundary in a particular direction,
 we use a line search with step size $0.05$.
We do not round DREBIN features to their
closest binary value---in contrast to what was done in Section~\ref{sec:exploration}. 
If we did, the line search would become too coarse-grained and yield less accurate measurements
of the distances between boundaries. An overview of the results
described below is given in
Figure~\ref{fig:main-distance-results} for 
the MNIST dataset. Further results are 
in Appendix~\ref{ap:dir-dis}. 
These include MNIST results with the $\ell_1$ norm and results on 
DREBIN. The observations made below hold
for those experiments as well. 

\textbf{Minimum Distance Measurements: } 
Distances to the decision boundary are indicated by filled 
black bars in Figure~\ref{fig:main-distance-results}.
As expected, the distance is smallest in adversarial directions. 
Moreover, decision boundaries are further away
in random directions than in directions between the different 
classes. This corroborates the observation
that random noise does not usually cause
misclassification~\cite{szegedy2013intriguing}.

\begin{figure}[t]
	\includegraphics[width=\textwidth]{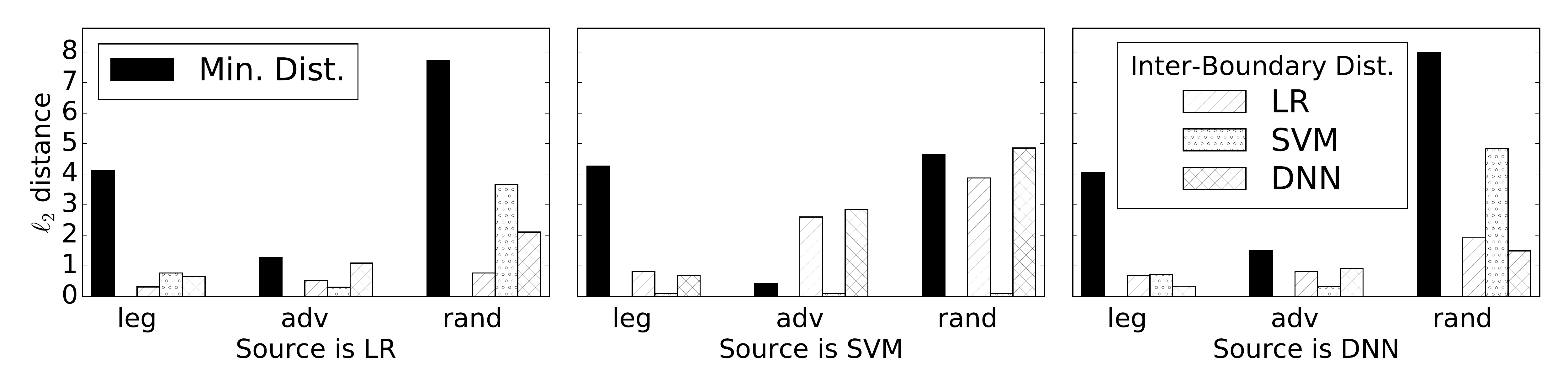}
	\vspace{-2em}
	\caption{Minimum distances and inter-boundary distances
		in three directions for MNIST models. 
		Each plot shows results for one source model 
		(Logistic Regression, Support Vector Machine, Deep Neural Network), 
		and all three classes of target models (one hatched bar per model class).
		Within each plot, bars are grouped 
		by direction (legitimate, adversarial and random). 
		The filled black bar shows the minimum distance to the 
		decision boundary for the source model.
		The adversarial search uses the FGM with 
		$\varepsilon=5$. For example, the left group in the 
		left plot shows that the minimal distance on the Logistic 
		Regression (LR) model in the legitimate direction is about $4$, 
		and that the distance between the LRs boundary and the boundaries 
		of other models in that direction is lower than $1$.\\[-1.5em]} 
	\label{fig:main-distance-results} 
\end{figure}

\textbf{Inter-Boundary Distances: } 
We now compare inter-boundary distances (hatched bars) 
with minimum distances (filled black bars). 
For most pairs of models and directions, the minimum distance
from a test input to the decision boundary is larger than the 
distance between the decision
boundaries of the two models in that direction. 
For the adversarial direction, this confirms our hypothesis 
on the ubiquity of transferability: for an adversarial example to
transfer, the perturbation magnitude needs to be only 
slightly larger than the \emph{minimum} perturbation required to fool the 
source model.\footnote{
The results with the SVM as source 
model are quantitatively different (the inter-boundary 
distances in the
adversarial directions are much larger). 
This is likely due to our SVM implementation
using a ``one-vs-rest'' strategy, which yields a 
different optimization problem for crafting adversarial examples.
}

\subsection{Impact of Adversarial Training on the Distance between Decision Boundaries}
\vspace*{-0.07in}

Surprisingly, it has been shown 
that adversarial examples can transfer even in the presence of explicit
defenses~\cite{goodfellow2014explaining,papernot2016practical}, such as 
distillation~\cite{papernot2016distillation} or adversarial 
training~\cite{goodfellow2014explaining}.
This suggests that these defenses do not significantly 
``displace'' a model's decision boundaries (and thus 
only marginally improve its robustness). Rather, 
the defenses prevent \emph{white-box} attacks 
because of \emph{gradient masking}~\cite{papernot2016towards}, i.e., 
they leave the decision boundaries in roughly the same location
but damage the gradient information 
used to craft adversarial examples. 
Hereafter, we investigate to 
what extent adversarial training displaces a model's decision 
boundaries.

We repeat the measurements of minimal and inter-boundary
distances with a pair of DNNs, one of which is adversarially
trained. We use the undefended model as source, and target the
adversarially trained model.
For all directions (legitimate, adversarial and random), we find 
that the inter-boundary distance between the source and target models
is increased compared to the baseline in 
Figure~\ref{fig:main-distance-results}.
However, this increase is too small 
to thwart transferability.
In the adversarial direction, the average inter-boundary distance
increases from $0.32$ (when the target is undefended) 
to $0.63$ (when the target is adversarially trained). The total
perturbation required to find a transferable input
(the minimum distance of $1.64$ summed with the inter-boundary
distance)
remains smaller than the 
adversarial perturbation of norm $5$. 
Thus, adversarial perturbations 
have sufficient magnitude to ``overshoot'' the decision boundary of 
the source model and still transfer to the adversarially trained 
model.

The transferability of adversarial examples from undefended 
models to adversarially trained models is studied extensively by 
Tramèr et al.~\cite{tramer2017ensemble}. The insights gained yield an 
improved adversarial training procedure, which increases the 
robustness of models to these black-box transferability-based attacks.

\section{Limits of Transferability}
\label{sec:limits}

Recall the hypothesis on the ubiquity of transferability formulated in 
the introduction:
\emph{
\begin{quoting}
If two models achieve low error for some task while also exhibiting 
low robustness to adversarial examples, adversarial examples crafted 
on one model transfer to the other.
\end{quoting} 
}
We first derive conditions on the data distribution $\mu$ under which a form of this 
hypothesis holds. However, we also show a counter-example: a specific task 
for which adversarial examples do not transfer between linear and 
quadratic models. This suggests that transferability is not an inherent property
 of non-robust ML models and that it may be possible to
 prevent black-box attacks (at least for some model classes) 
 despite the existence of adversarial examples.\footnote{
We consider adversaries that attack a target model by transferring 
adversarial examples crafted on a locally trained model. Note that 
if an adversary can mount a model theft 
attack~\cite{papernot2016practical, tramer2016stealing} 
that recovers a sufficiently close approximation of the target model,
 then any model vulnerable in a white-box setting can be attacked.}

\subsection{Sufficient Conditions for Transferability}
\label{ssec:sufficient}
\vspace*{-0.07in}

We present \emph{sufficient conditions} under which 
a specific type of adversarial perturbation transfers from input space 
to richer latent feature spaces in binary classification tasks. 
These perturbations are \emph{model agnostic}; they
shift data points in the direction given by the 
\emph{difference between the intra-class means}.
We consider binary classifiers $f(\vec{x}) = \vec{w}^\top \phi(\vec{x}) + b$ defined as the composition of 
an arbitrary feature mapping $\phi(\vec{x})$ and a linear classifier.
These include
linear models, polynomial models and feed-forward neural networks 
(where $\phi$ is defined by all but the last layer of the network).

Under natural assumptions, any linear classifier 
with low risk will be fooled by perturbations in the direction
of the difference in class means (this follows from a result by 
Fawzi et al.~\cite{fawzi2015analysis}).
Moreover, for classifiers over a richer latent feature set 
(e.g. quadratic classifiers), we relate robustness to these 
perturbations to the extent in which \emph{the direction between the 
class means in input space is preserved in feature space}. 
This is a property solely of the data 
distribution and feature mapping considered.

\paragraph{Model-Agnostic Perturbations.}

For a fixed feature mapping $\phi$, we consider the perturbation 
(in feature space) that shifts data points in the direction given 
by the difference in intra-class means. Let
\vspace*{-.03in}
\begin{equation}
\label{eq:mean-dir}
\vec{\delta}_\phi \coloneqq \frac12 \cdot (\meanfp - \meanfm) \;.
\end{equation}
where $\Dm$ and $\Dp$ respectively denote the data distributions of 
the positive and negative class.
For an input $(\vec{x},y)$, we define the feature-space perturbation 
$\vec{r}_\phi \coloneqq - \epsilon \cdot y \cdot \hat{\vec{\delta}}_\phi$, where 
$\hat{\vec{v}}$ denotes a non-zero vector $\vec{v}$ normalized to unit norm.
For large $\epsilon$, the feature-space perturbation 
$\vec{r}_\phi$ will fool any model $f$ if the weight 
vector $\vec{w}$ is aligned with the difference in class means 
$\vec{\delta}_\phi$, i.e. we have: 
\begin{equation}
\label{eq:assumption}
\Delta \coloneqq \hat{\vec{w}}^\top \hat{\vec{\delta}}_\phi > 0 \;.
\end{equation}
For some tasks and model classes, we empirically show that 
the assumption in \eqref{eq:assumption} must hold for $f$ to obtain low error.
Following~\cite{fawzi2015analysis}, we
need $\epsilon \approx \norm{\vec{\delta}_\phi}_2$ on average to fool $f$ 
(see Appendix~\ref{apx:proof-perturb}). 
For tasks with close class means 
(e.g., many vision tasks~\cite{fawzi2015analysis}), the 
perturbation is thus small.

Do model-agnostic perturbations in 
\emph{input space} transfer to models over richer latent feature 
spaces? Let $\vec{\delta}$ and $\vec{r}$ denote the difference 
in means and model-agnostic perturbation in input space, i.e.,
\begin{equation}
\label{eq:perturb-inp}
\vec{\delta} \coloneqq \frac12\cdot (\meanp - \meanm) \quad, \quad \vec{r} \coloneqq -\epsilon \cdot y \cdot \hat{\vec{\delta}} \;.
\end{equation}
If the weight vector $\vec{w}$ is aligned with the 
difference in class means in feature space $\vec{\delta}_\phi$, 
a sufficient condition for transferability is that $\vec{r}$ gets mapped 
to a perturbation in feature space that is closely aligned 
with $\vec{\delta}_\phi$, 
i.e., $\phi(\vec{x}+\vec{r}) \approx \phi(\vec{x}) + \vec{r}_\phi$.
Specifically, consider the orthogonal decomposition
\begin{equation}
\label{eq:alpha}
\phi(\vec{x}+\vec{r}) - \phi(\vec{x}) = \alpha \cdot \vec{\delta}_\phi + \beta \cdot \vec{\delta}_\phi^\bot \;,
\end{equation}
where $\vec{\delta}_\phi^\bot$ is orthogonal to $\vec{\delta}_\phi$ and of equal 
norm (i.e., $||\vec{\delta}_\phi^\bot||_2 = ||\vec{\delta}_\phi||_2$).
Intuitively, we want $\alpha$ to be 
large, and $\beta$ small. If so, we say the mapping 
$\phi$ is \emph{pseudo-linear} in $\vec{r}$. Our main theorem (which we prove 
in Appendix~\ref{apx:proof-perturb}) is as follows:

\begin{theorem}
	\label{thm:perturb-full}
	Let $f(\vec{x}) = \vec{w}^\top \phi(\vec{x}) + b$, where $\phi$ is a fixed feature 
	mapping. Let $\vec{r}$ be defined as in~\eqref{eq:perturb-inp}. 
	If the angle $\Delta \coloneqq \hat{\vec{w}}^\top \hat{\vec{\delta}}_\phi$ 
	is not too small, i.e., $\Delta \gg \Pr_{\Dist} (\mathtt{sign}(f(\vec{x})) \neq y)$, and 
	$\Exp_{\Dist}(- y \alpha) \geq 1 + \frac{1-\Delta}{\Delta} \Exp_{\Dist}(\abs{\beta})$, 
	then adversarial examples $\vec{x}+\vec{r}$ will fool $f$ with 
	non-zero probability.
\end{theorem}

\paragraph{Experiments on MNIST.}

We empirically assess the validity of our analysis on the binary 
task of distinguishing $3$'s and $7$'s from the MNIST dataset.
We use linear models, quadratic models 
($f(\vec{x}) = \vec{x}^\top A \vec{x} + b$, where $A$ is a symmetric matrix), DNNs 
and CNNs. All models are interpreted as having the form 
$f(\vec{x}) = \vec{w}^\top \phi(\vec{x}) + b$ for some appropriate mapping $\phi$.

All models satisfy the assumption in Equation~\eqref{eq:assumption}: 
the linear model in 
feature space is strongly aligned with the difference in class means.
In fact, we find that any 
linear or quadratic model that does not satisfy this assumption 
has low accuracy for this task.
We use the mean-shift perturbation $\vec{r}$ defined in~\eqref{eq:perturb-inp} with $\epsilon 
= 4$. Examples of perturbed inputs are in Figure~\ref{fig:mean-shifts}; the 
correct labels are self-evident despite the appearance of a faint ``ghost'' 
digit in some images.

\begin{figure}[h]
	\centering
	\vspace{-0.45em}
	\includegraphics[width=0.6\textwidth]{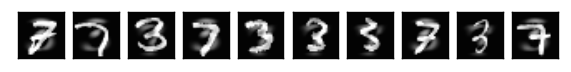}
	
	\vspace{-0.5em}
	\caption{MNIST digits perturbed by adding the difference in 
		class means.\\[-0.45em]}
	\label{fig:mean-shifts}
\end{figure}

Table~\ref{tab:mean-shift} reports the accuracy of each model on clean data 
and perturbed inputs. All models are 
partially fooled by these small model-agnostic perturbations. 
In addition, we report whether each model
satisfies the technical conditions of Theorem~\ref{thm:perturb-full} 
(i.e., the \emph{pseudo-linearity} of mapping $\phi$ with respect 
to the perturbation $\vec{r}$). 
Since the quadratic feature mapping satisfies these conditions, 
\emph{any} quadratic classifier with high accuracy (and thus positive 
alignment between $\vec{w}$ and $\vec{\delta}_\phi$) will be  
vulnerable to these perturbations. Thus, 
\emph{any accurate pair of linear and quadratic classifiers for this 
task will exhibit transferability}. This is a special case 
of the hypothesis formulated at the beginning of this section.
The pseudo-linearity condition does not hold for the CNN: the 
perturbation direction is not sufficiently preserved 
in feature space for our result to hold unconditionally. However, 
our bounds are (expectedly) somewhat loose and we find that the 
CNN still misclassifies $24\%$ of perturbed samples. 

To generalize this approach to multi-class settings, we  
can define perturbations through pairwise differences in means: $
\Exp_{\Dist_i}[\vec{x}] - \Exp_{\Dist_j}[\vec{x}]$,
where $\vec{x}$ has label $y=i$, and $j\neq i$ is a chosen target 
class. We can also use the $\ell_\infty$ norm, by setting 
$\vec{r} \coloneqq - \epsilon \cdot 
\mathtt{sign}(\Exp_{\Dist_i}[\vec{x}] - \Exp_{\Dist_j}[\vec{x}])$.
With this adversarial perturbation (for $\epsilon=0.3$) the same 
model architectures trained on the full MNIST dataset 
attain accuracy between $2\%$ (for the linear 
model) and $66\%$ (for the CNN).

\begin{table}
\caption{Results for the model agnostic adversarial perturbation 
	that shifts points in the direction of the difference in class 
	means. We used $\epsilon = 4$, for $\vec{r}$ as defined 
	in~\eqref{eq:perturb-inp}. The angle $\Delta$ is measured 
	between the weights $\vec{w}$ of the final linear layer and the 
	difference in class means in feature space, $\vec{\delta}_\phi$.\\[-0.5em]
}
\centering
\begin{tabular}{@{} l c c c c @{}}
\textbf{Model} & \textbf{Acc.} & \textbf{Acc. on shifted data} & ${\Delta}$ & $\phi$ satisfies pseudo-linearity in $\vec{r}$ \\
\toprule
Linear & $98.2\%$ & $39\%$ & $0.55$ & Yes\\
Quadratic & $99.3\%$ & $51\%$ & $0.51$ & Yes\\
\midrule
DNN & $99.3\%$ & $43\%$ & $0.68$ & Yes\\
CNN & $99.9\%$ & $76\%$ & $0.79$ & No \\
\midrule 
\end{tabular}
\label{tab:mean-shift}
\\[-0.5em]
\end{table}

\paragraph{Experiments on DREBIN.}
We evaluated the same model-agnostic perturbations on the DREBIN 
dataset. It is noteworthy that on DREBIN, the class mean is less 
semantically meaningful than on MNIST (i.e., the mean feature vector 
of all malware applications is unlikely to even belong to the input 
space of valid programs). Yet, the same perturbation (followed by 
rounding to binary features) reliably fools both a linear model and 
a DNN. By modifying fewer than $10$ of the $1000$ binary features on 
average, the accuracy of the linear model drops to $65\%$ and that 
of the DNN to $72\%$.

\subsection{XOR Artifacts}
\label{ssec:xor}

The previous analysis shows that transferability can be inherent 
if models are defined on (or learn) feature representations that 
preserve the non-robustness of input-space features. We now show that if 
models were to learn very different sets of (possibly non-robust) 
features, transferability need not hold.
As a thought-experiment, consider two models 
trained on \emph{independent subsets} of the input features. 
In this setting, 
adversarial examples crafted to specifically mislead one of the models 
would not transfer to the other. However, it is 
expected that adversarial examples crafted on a model trained over 
the full feature set would still transfer to both models.
Hereafter, we ask a more realistic question: for a given task, can 
models with access to the \emph{same set of input features}  
learn representations for which adversarial examples do not transfer 
to one another---even when these models are
not robust to their own adversarial examples.
We provide a simple example of a task for which this is true.

We first train linear and quadratic models on the  
MNIST 3s vs.~7s task. Both models get over $98\%$ accuracy. 
FGM examples (with $\norm{\vec{r}}_2 = 4$) transfer between the models 
at a rate above $60\%$.
We then consider a variant of this dataset with pixels taking 
values in $\Reals^d$, and where a special 
``XOR artifact'' is added to images: for two of the center pixels, 
images of 3s contain a negative XOR (i.e., one of the pixels has 
positive value, the other negative), and images of 7s contain 
a positive XOR (both pixels have either positive or negative value). 
Some examples are shown in Figure~\ref{fig:mnist-xor}. 
Intuitively, this alteration will be ignored by the linear model 
(the artifact is not linearly separable) but not by the quadratic 
model as the XOR product is a perfect predictor.

\begin{figure}[h]
	\centering
	\includegraphics[width=0.6\textwidth]{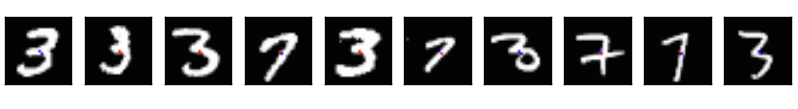}
	
	\vspace{-0.5em}
	\caption{MNIST digits with an additional XOR artifact. 
		Best viewed zoomed in and in color. Blue pixels are positive, 
		red pixels are negative.}
	\label{fig:mnist-xor}
\end{figure}

Both models trained on this data remain vulnerable to small, 
\emph{but different}, adversarial perturbations. The linear model 
is fooled by standard FGM examples and the quadratic model is 
fooled by flipping the sign of the two center pixels. 
Neither of these perturbations transfers to the other model. 
Fooling both models simultaneously 
requires combining the perturbations for each individual model.
Our experiment on this artificially altered dataset demonstrates 
how different model classes may learn very different concepts that, 
while strongly predictive, have little correspondence with our 
visual perception of the task at hand. It is interesting to consider 
whether any real datasets might exhibit a similar property, and 
thus result in learned models exhibiting low transferability. 

\section*{Acknowledgments}

We thank Ben Poole and Jacob Steinhardt for feedback on early versions of this work.
Nicolas Papernot is supported
by a Google PhD Fellowship in Security. 
Research was supported in part by the Army Research Laboratory,
under Cooperative Agreement Number W911NF-13-2-0045 (ARL Cyber Security
CRA), and the Army Research Office under grant W911NF-13-1-0421. 
The views and conclusions contained in this document are those of the
authors and should not be interpreted as representing the official policies,
either expressed or implied, of the Army Research Laboratory or the U.S.
Government. The U.S.\ Government is authorized to reproduce and distribute
reprints for government purposes notwithstanding any copyright notation hereon.

\newpage
\appendix
\section{Neural Network Architectures}

\begin{table}[h!]
\caption{Neural network architectures used in this work. 
	Conv: convolutional layer, FC: fully connected layer.\\}
\centering
\begin{tabular}{@{} c c c@{}c @{}}
\multicolumn{4}{c}{\textbf{Model ID}}  \\
\toprule
A & B & C \\
\midrule
Conv(64, 5, 5) + Relu & Dropout(0.2) & \multirow{2}{*}{$\left[\shortstack[c]{FC(300) + Relu \\ Dropout(0.5) \\[-1.2em]}\right]$} & \multirow{2}{*}{$\times 4$} \\
Conv(64, 5, 5) + Relu & Conv(64, 8, 8) + Relu & \\
Dropout(0.25) & Conv(128, 6, 6) + Relu &  FC + Softmax\\
FC(128) + Relu & Conv(128, 5, 5) + Relu & \\
Dropout(0.5) & Dropout(0.5) & \\
FC + Softmax & FC + Softmax &\\
 \midrule
\end{tabular}

\label{table:mnist_archis}

\end{table}

\section{Finding Multiple Adversarial Directions}
\label{apx:multi-adv}

Our first two techniques iteratively solve some optimization problem 
with an additional constraint enforcing orthogonality of the solutions.
The latter two estimate the dimensionality of the adversarial space more 
directly, using first-order approximations of the model's output or loss.

\paragraph{Second-Order Approximations.}
Consider an analog of the FGM with a second-order approximation of the 
loss function. That is, we want a perturbation $\vec{r}$ that solves 
$
\max_{\|\vec{r}\| \leq \epsilon}\ \vec{g}^\top \vec{r} + \frac12 \vec{r}^\top H \vec{r}
$,
where $\vec{g}=\nabla_{\vec{x}} J(\vec{x}, y)$ is the gradient and 
$H=\nabla^2_{\vec{x}}
J(\vec{x}, y)$ is the Hessian of the loss function.

We can find orthogonal perturbations as follows: Given a first 
solution $\vec{r}_1$ to the above optimization problem, we 
substitute $\vec{g}$ and $H$ by $\vec{g}' \coloneqq P\vec{g}$ and 
$H' \coloneqq P^\top H P$, where $P$ is a projection matrix onto the 
space orthogonal to $\vec{r}_1$. The solution $\vec{r}_2$ to the 
updated problem is then orthogonal to $\vec{r}_1$.

\paragraph{Convex Optimization for Piecewise Linear Models.}

An entirely different approach to finding adversarial examples was 
proposed by Bastani et al.~\cite{bastani2016measuring} for the 
(ubiquitous) case where the model $f$ is a piecewise linear function 
(here we identify the output of $f$ with the vector of unnormalized 
log-probabilities preceding the softmax output). 
Instead of relying on a linear \emph{approximation} of the model's
output, their idea is to restrict the search of adversarial 
perturbations for an input $\vec{x}$ to the convex region $Z(\vec{x})$ 
around $\vec{x}$ on which $f$ is linear.

Let $\vec{x}$ be a point that $f$ classifies as class $s$ 
(i.e., $\argmax f_i(\vec{x}) = s$). 
Further let the ``target'' $t$ be the class to which 
$f$ assigns the second-highest score. 
The approach from~\cite{bastani2016measuring} 
consists in finding a perturbation $\vec{r}$ as a solution to a convex 
program with the following constraints, denoted $C(\vec{x})$:

\begin{itemize}
	\item $\vec{x}+\vec{r} \in Z(\vec{x})$, which can be expressed as a set of linear constraints (see~\cite{bastani2016measuring} for more details).
	\item $\|\vec{r}\| \leq \epsilon$ (although Bastani et al. focus on the $l_\infty$ norm, any convex norm can be used).
	\item $f_t(\vec{x}+\vec{r}) < f_s(\vec{x}+\vec{r})$.
\end{itemize}

If all constraints are satisfied, $\vec{x}+\vec{r}$ must be misclassifed 
(possibly as another class than the target $t$).

We suggest an extension to this LP for finding orthogonal adversarial 
directions. After finding a
solution $\vec{r}_1$ with constraint set $C(\vec{x})$, we extend the constraints 
as $C'(\vec{x}) \coloneqq C(\vec{x}) \wedge (\vec{r}^\top \vec{r}_1 = 0)$, and solve the 
convex program again to find a new solution $\vec{r}_2$ orthogonal to $\vec{r}_1$.
This process is iterated until the constraints are unsatisfiable.

\paragraph{Independent JSMA.}
For discrete features, such as in the DREBIN dataset, it is 
difficult to find (valid) orthogonal perturbations using the 
previous methods. 

Instead, we propose a simple non-iterative variant of the 
Jacobian-based Saliency Map Attack (JSMA) 
from~\cite{papernot2016limitations}. 
This iterative method evaluates the Jacobian matrix of 
the model $f$ and ranks features in a saliency map that encodes 
the adversarial goal~\cite{papernot2016limitations}.
A few input features (typically one or two) with high saliency
scores are perturbed (i.e., increased to their maximum value) 
at each iteration, before repeating the process with the modified input.

We compute the saliency map as in~\cite{papernot2016limitations}
and rank the features by their saliency scores. Given a target 
dimensionality $k$ and a maximal perturbation budget $B$, we group 
the most salient features into $k$ bins of at most $B$ features 
each, such that the sum of saliency scores in each bin are roughly 
equal (we greedily add new features into the non-full bin with 
lowest total score). As the $k$ sets of features are independent by 
construction, they naturally yield orthogonal perturbations. 
We repeat this process for multiple values of $k$ to estimate the 
dimensionality of the adversarial space.

\subsection{Prefiltered Evaluation of the Transfer Rate}
\label{ssec:transferability-def}

To evaluate the transferability from one model to another, we report
accuracy in ``the prefiltered case,'' as done by \cite{kurakin2016adversarial}.
That is, we report the accuracy of the target model on adversarial examples
that meet three criteria: 1) the original example did not fool the first model,
and 2) the adversarial example successfully fools the first model.
Departing from the metric used by \cite{kurakin2016adversarial}, we also
require that 3) the original example did not fool the target model.
This evaluation metric focuses on the transferability property rather than
effects like natural inaccuracy of the models or failure of the adversarial
example construction process.
Accuracy in this prefiltered case can be considered the "transfer rate" of
the adversarial examples.

\subsection{Experiments}

\begin{table}
	\caption{Comparison of techniques for finding multiple orthogonal 
		perturbations on MNIST. We report: (1) the transfer rate of 
		each method when used to find a single adversarial direction, 
		(2) the average number of (successful) orthogonal adversarial 
		perturbations found on the source model f$_{\text{src}}$ and 
		(3) the average number of these perturbations that transfer to 
		model f$_{\text{target}}$.\\}
	
	\centering
	\begin{tabular}{@{} l l c r r@{} }
		\textbf{Model} & \textbf{Method} & \textbf{Transfer Rate} & \textbf{$\#$ of dirs on f$_{\text{src}}$} & \textbf{$\#$ of dirs on f$_{\text{target}}$}\\
		\toprule
		\multirow{3}{0.4in}{DNN} & Second-Order & $95\%$ & $1.00$ & $0.95$\\
		& Convex Opti. & $17\%$ & $\ 2.58$ & $\ 0.45$\\
		& GAAS & $94\%$ & $44.28$ & $24.87$ \\
		\midrule
		CNN & GAAS & $69\%$ & $\ 15.18$ & $\ 2.24$ \\
		\midrule
	\end{tabular}
	\label{tab:multi-adv}
\end{table}

\paragraph{Second-Order Approximations.}
This attack is not interesting for rectified linear units, because
the Hessian of the logits with respect to the inputs
is zero (or undefined) everywhere.
In order to study this attack with a nonzero Hessian, 
we use hyperbolic tangent rather than rectified linear units
in the source model.
Yet, we still find that $H$ has only small entries and that the 
optimal solution to the quadratic problem is close to perfectly 
aligned with the gradient. The best solution that is orthogonal to 
the gradient is non-adversarial. Thus, the second-order approximation 
is of little use in finding further adversarial directions.

\paragraph{Convex Optimization.} On the source model f$_{\text{src}}$, the LP  yields an
adversarial example for $93\%$ of the test inputs. 
By solving the LP iteratively, we find a little over $2$ orthogonal 
perturbations on average, and $47$ for the ``most vulnerable'' input. 
As we restrict the search to a region where the model is linear, 
these orthogonal perturbations span a dense subspace. That is, any 
perturbation of norm $\epsilon$ in the spanned subspace is also 
adversarial.

However, the
transfer rate for these perturbations is low ($17\%$ compared to 
$95\%$ for the FGM). When considering the full adversarial subspace, 
the transferability increases only marginally: for $70\%$ of the 
inputs, none of the successful perturbations on model 
f$_{\text{src}}$ transfers to model f$_{\text{target}}$.
Thus, although the LP approach may find smaller adversarial 
perturbations and in more directions than methods such as 
FGM~\cite{bastani2016measuring}, we also find that it is less 
effective at yielding transferable perturbations.\footnote{
	The approach of Bastani et
	al.~\cite{bastani2016measuring} sometimes fails to produce adversarial
	examples. For instance, for architecture C in Table~\ref{table:mnist_archis} and
	perturbations $\norm{\vec{r}}_2 \leq 5$, the obtained LP is solvable 
	for $0\%$ of tested inputs, whereas the FGSM always yields a 
	successful adversarial example.}

\paragraph{Independent JSMA.}
We evaluated the JSMA variant on the DREBIN dataset, for the same 
DNN models as above. With a maximal budget of $10$ flipped features 
per perturbation, we find at least one successful perturbation for 
$89\%$ of the inputs, and $42$ successful independent perturbations 
on average. Of these, $22$ transfer to the target model on average. 

Although we did not verify that all these perturbations truly retain 
the (non)-malicious nature of the input program, we note that 
flipping a \emph{random} subset of $10$ input features does not 
cause the models to misclassify at a higher rate than on clean data.
As our pre-processing of the DREBIN dataset selects those 
features that have a large impact on classification (and thus are 
often useful to adversaries), we don't expect the number of 
independent perturbations to increase significantly for variants of 
DREBIN with larger feature dimensionality.

\section{Proofs}

\subsection{Proof of Lemma~\ref{lemma:orthogonal}}
\label{apx:proof-ortho}

We begin with the upper bound. Let 
$\hat{\vec{r}}_i \coloneqq \frac{\vec{r}_i}{\norm{\vec{r}_i}_2}$. 
Note that 
if $\vec{r}_1, \vec{r}_2, \dots, \vec{r}_k$ are orthogonal, we have
\[
\|\vec{g}\|_2^2  \geq \sum_{i=1}^k \abs{\vec{g}^\top \hat{\vec{r}}_i}^2 \geq \sum_{i=1}^k \frac{\abs{\vec{g}^\top \vec{r}_i}^2}{\norm{\vec{r}_i}_2^2} \geq k \cdot \alpha^2 \cdot \norm{\vec{g}}_2^2 \;,
\]
which implies $k \leq \left\lfloor \frac{1}{\alpha^2}\right\rfloor$. 
If $\alpha < \frac{1}{\sqrt{d}}$, the bound $k \leq d$ is trivial.

We prove the lower bound with a concrete construction of the 
vectors $\vec{r}_i$. Let $\vec{e}_1, \vec{e}_2, \dots, \vec{e}_d$ 
denote the basis vectors in $\Reals_d$. 
Let $k = \left\lfloor \frac{1}{\alpha^2}\right\rfloor \leq d$ and 
denote by $R$ a rotation matrix that satisfies 
$R \vec{g} = \norm{\vec{g}}_2 \cdot \vec{e}_1$ (i.e., $R$ rotates 
vectors into the direction given by $\vec{g}$). 
Let $\vec{z}\coloneqq \sum_{i=1}^k k^{-\frac12}\cdot \vec{e}_i$, 
and let $S$ be the rotation matrix that satisfies 
$S \vec{z} = \vec{e}_1$. Then, it is easy to see that 
$Q\coloneqq S^\top R$ is a rotation matrix that satisfies 
$Q\vec{g} = \norm{\vec{g}}_2 \cdot \vec{z}$.

The vectors $\vec{r}_i \coloneqq Q^\top \vec{e}_i$ for 
$1 \leq i \leq k$ (these are the first $k$ rows of $Q$) 
are orthonormal, and satisfy 
$\vec{g}^\top \vec{r}_i 
= \norm{\vec{g}}_2 \cdot \vec{z}^\top \vec{e}_i 
= \norm{\vec{g}}_2 \cdot k^{-\frac12} \geq \norm{\vec{g}}_2 \cdot \alpha$.

\subsection{Proof of Theorem~\ref{thm:perturb-full}}
\label{apx:proof-perturb}

We begin by deriving the average $\ell_2$-norm of the perturbation 
$\vec{r}_\phi$ in the latent feature space, so that $f$ misclassifies \emph{all} 
inputs. Without loss of generality, we assume that 
$b < \max_{\vec{x} \in \inpspace} \lvert \vec{w}^\top \phi(\vec{x}) \rvert$ (i.e., the sign 
of the classifier is not constant). The expected $0$-$1$ loss of $f$ 
over some data distribution $\Dist$ is denoted 
$L(f) \coloneqq \Pr_{\Dist} (\mathtt{sign}(f(\vec{x})) \neq y)$.

\begin{lemma}
\label{lemma:perturb}
Let $f=\vec{w}^\top \phi(\vec{x}) + b$ with $L(f) < \frac12$ and 
$\Delta\coloneqq \hat{\vec{w}}^\top \hat{\vec{\delta}}_\phi > 0$.
For any $\vec{x} \in \inpspace$, let $\vec{r}_\phi$ be the smallest 
(in $\ell_2$-norm) perturbation aligned with $\vec{\delta}_\phi$, such 
that $\mathtt{sign}(\vec{w}^\top (\phi(\vec{x})+\vec{r})) \neq y$. 
Then,
\[
\Exp_{\Dist}(\norm{\vec{r}_\phi}_2) \leq \norm{\vec{\delta}_\phi}_2 
+ \frac{4\cdot L(f)\cdot \max_{\vec{x} \in \inpspace}\norm{\phi(\vec{x})}_2}{\Delta} \;.
\]
\end{lemma}

\begin{proof}
Recall that $\vec{\delta}_\phi = \frac12 \cdot (\meanfp - \meanfm)$. 
If $\vec{x}$ is misclassified by $f$, then $\norm{\vec{r}_\phi}_2 = 0$. 
Otherwise, it is easy to see that 
\[
\norm{\vec{r}_\phi}_2 = 
\displaystyle\frac{\abs{f(\vec{x})}}{\vec{w}^\top \vec{\delta}_\phi} 
\cdot \norm{\vec{\delta}_\phi}_2 \;.
\] 

A simple adaption of the analysis of Fawzi et 
al.~\cite{fawzi2015analysis} (for the case of linear models) yields
\[
\Exp_{\Dist}[\abs{f(\vec{x})}] \leq \vec{w}^\top \vec{\delta}_\phi + 4\cdot L(f)\cdot \norm{\vec{w}}_2 \cdot \max_{\vec{x} \in \inpspace} \norm{\phi(\vec{x})}_2\;.
\]

Combining these results, we obtain
\begin{align*}
\Exp_{\Dist}(\norm{\vec{r}_\phi}_2) 
&\leq \norm{\vec{\delta}_\phi}_2 + 4 \cdot L(f) \cdot \max_{\vec{x} \in \inpspace} \norm{\phi(\vec{x})}_2 \cdot \frac{\norm{\vec{w}}_2\cdot \norm{\vec{\delta}_\phi}_2}{\vec{w}^\top \vec{\delta}_\phi} \\
& = \norm{\vec{\delta}_\phi}_2 + \frac{4 \cdot L(f) \cdot \max_{\vec{x} \in \inpspace} \norm{\phi(\vec{x})}_2}{\Delta}
\end{align*}
\end{proof}

Lemma~\ref{lemma:perturb} is an extension of the result of Fawzi et 
al.~\cite{fawzi2015analysis}, who show that if the difference in 
class means is small, any linear model with low risk can be fooled 
by small perturbations along the direction of $\vec{w}$. Because we are 
interested in the simple ``model-agnostic'' perturbation that 
directly follows the difference of means, the magnitude of the 
perturbation further depends on the alignment $\Delta$ between $\vec{w}$ 
and $\vec{\delta}_\phi$, which is positive according to 
the assumption in Equation~\eqref{eq:assumption}.

We now consider transferability of the perturbation $\vec{r}$ in input 
space. Recalling~\eqref{eq:alpha}, we can write
\begin{align*}
\vec{w}^\top \phi(\vec{x}+\vec{r}) &= \vec{w}^\top \left( \phi(\vec{x}) + \alpha \cdot \vec{\delta}_\phi + \beta \cdot \vec{\delta}_\phi^\bot \right)
\end{align*}

Assume $\Delta \gg L(f)$ and suppose first that $\beta = 0$. 
Then, according to Lemma~\ref{lemma:perturb}, on average, the minimal 
value of $(-y \cdot \alpha)$ to get $f$ to misclassify is $\approx 1$. 
However, the contribution of $\beta\cdot \vec{\delta}_\phi^\bot$ has 
to be taken into account. 
Note that $\hat{\vec{w}}^\top \hat{\vec{\delta}}_\phi^\bot \leq 1-\Delta$. 
Thus, we have
\[
\abs{\beta \cdot \vec{w}^\top \vec{\delta}_\phi^\bot} \leq \abs{\beta} \cdot \norm{\vec{w}}_2 \cdot \norm{\vec{\delta}_\phi} \cdot (1-\Delta) \;.
\]

In the worst case, the absolute magnitude of $\alpha$ should be 
increased by $\frac{1-\Delta}{\Delta} \abs{\beta}$ to ensure that 
any non-adversarial contribution of the orthogonal component of the 
perturbation is canceled out. Thus, in expectation, we need 
$-y\cdot \alpha \geq 1 + \frac{1-\Delta}{\Delta} \abs{\beta}$ to 
have $f$ misclassify with non-zero probability. Note that this bound
is expected to be loose in practice, as the orthogonal component 
$\vec{\delta}_\phi^\bot$ is unlikely to be maximally aligned 
with $\vec{w}$.

\newpage
\section{Appendix: Additional Distance Measurements}
\label{ap:dir-dis}

We provide additional measurements of the similarity in models' 
decision boundaries.
Inter-boundary distances computed on the MNIST models with the 
$\ell_1$ norm are reported in Figure~\ref{subfig:mnist-l1-distances}. 
Results for the $\ell_2$ norm were in 
Figure~\ref{fig:main-distance-results}. 
Analog results for the DREBIN models
are reported in Figure~\ref{subfig:drebin-l1-distances} 
and~\ref{subfig:drebin-l2-distances}.

\begin{figure}[h] 
	\begin{subfigure}[b]{\textwidth}
		\includegraphics[width=\textwidth]{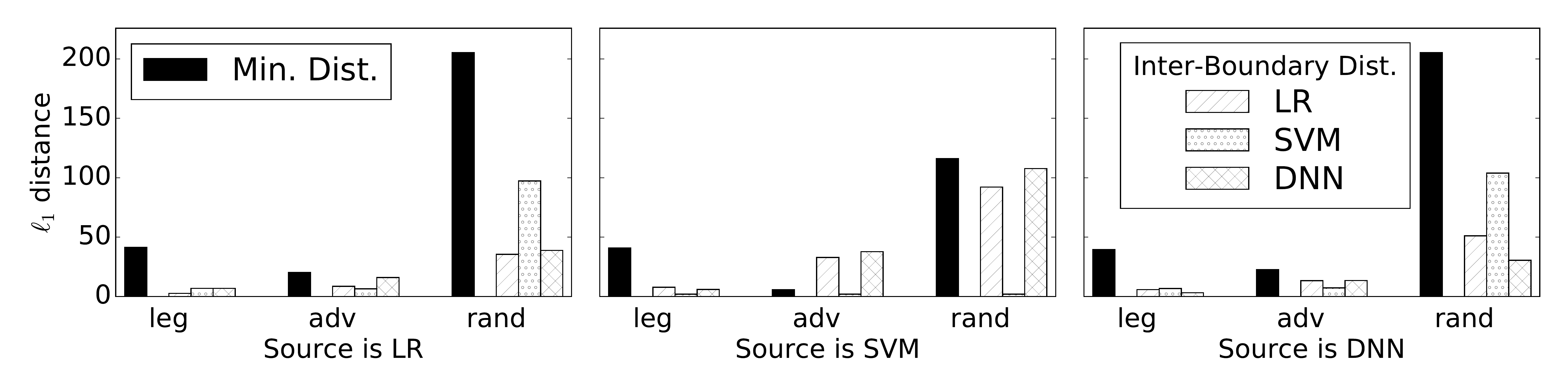}
		\caption{Distances measured using the $\ell_1$ norm for the MNIST models. 
			The adversarial search uses $\varepsilon=80$.} 
		\label{subfig:mnist-l1-distances} 
	\end{subfigure}
	\begin{subfigure}[b]{\textwidth}
		\includegraphics[width=\textwidth]{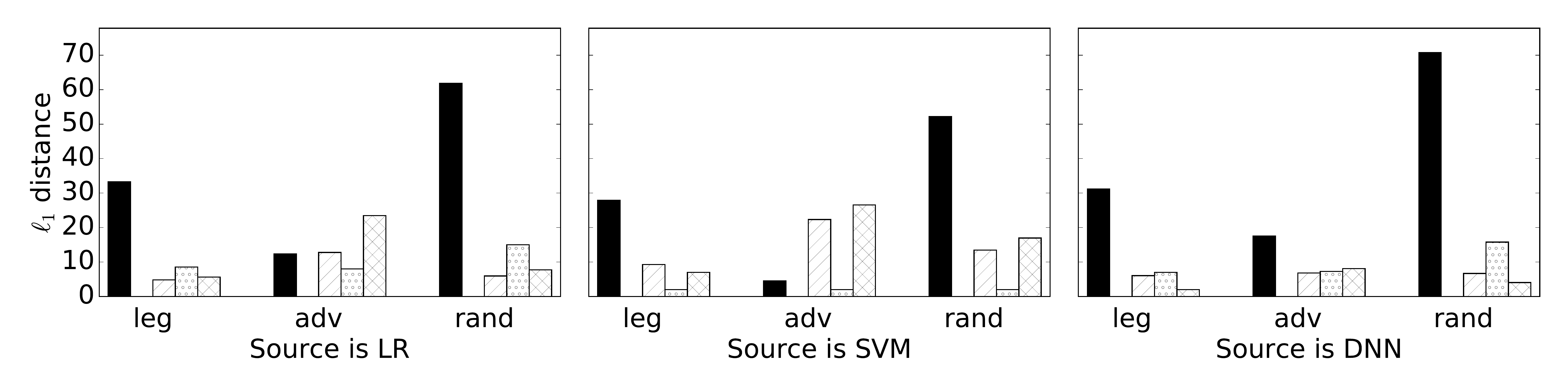}
		\caption{Distances measured using the $\ell_1$ norm for the DREBIN models. 
			The adversarial search uses $\varepsilon=80$.} 
		\label{subfig:drebin-l1-distances} 
	\end{subfigure}
	\begin{subfigure}[b]{\textwidth}
		\includegraphics[width=\textwidth]{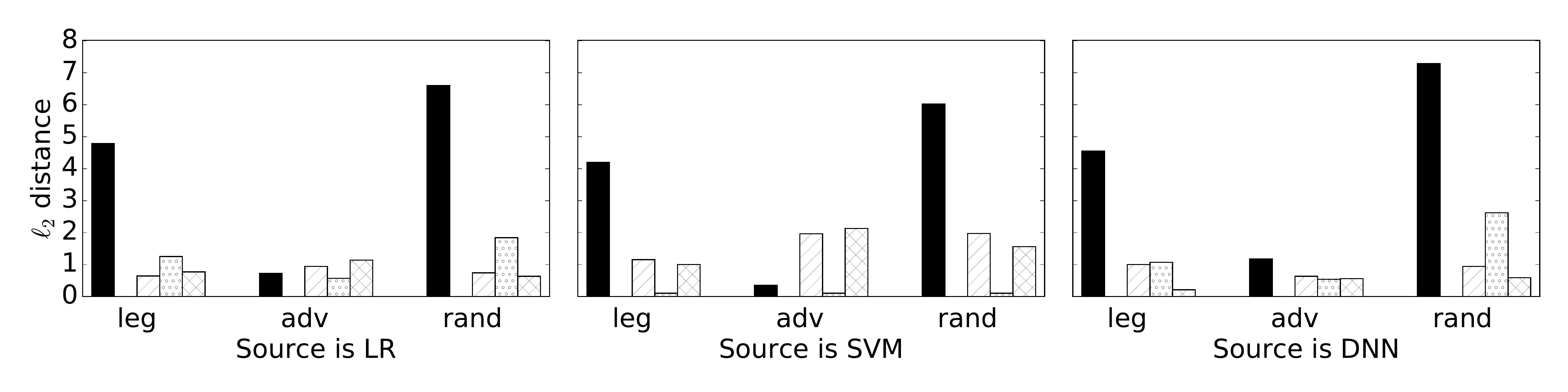}
		\caption{Distances measured using the $\ell_2$ norm for the DREBIN models. 
			The adversarial search uses $\varepsilon=5$.} 
		\label{subfig:drebin-l2-distances} 
	\end{subfigure}
	\caption{Minimum distances and inter-boundary distances
		in three directions between models trained on MNIST and DREBIN. 
		Each plot shows results for one source model, and all three
		classes of target models (one hatched bar per model class).
		Within each plot, bars are grouped 
		by direction (legitimate, adversarial and random). 
		The filled black bar shows the minimum distance to the decision 
		boundary, computed on the source model.
		The adversarial search uses the FGM in either the $\ell_1$ or
		$\ell_2$ norm.}
	\label{fig:learning-approximators}
\end{figure}


\begin{thebibliography}{10}	
	\bibitem{arp2014drebin}
	Daniel Arp, Michael Spreitzenbarth, Malte Hubner, Hugo Gascon, Konrad Rieck,
	and CERT Siemens.
	\newblock Drebin: Effective and explainable detection of android malware in
	your pocket.
	\newblock In {\em NDSS}, 2014.
	
	\bibitem{bastani2016measuring}
	Osbert Bastani, Yani Ioannou, Leonidas Lampropoulos, Dimitrios Vytiniotis,
	Aditya Nori, and Antonio Criminisi.
	\newblock Measuring neural net robustness with constraints.
	\newblock In {\em NIPS}, pages 2613--2621, 2016.
	
	\bibitem{behzadan2017vulnerability}
	Vahid Behzadan and Arslan Munir.
	\newblock Vulnerability of deep reinforcement learning to policy induction
	attacks.
	\newblock {\em arXiv preprint arXiv:1701.04143}, 2017.
	
	\bibitem{biggio2013evasion}
	Battista Biggio, Igino Corona, Davide Maiorca, Blaine Nelson, Nedim
	{\v{S}}rndi{\'c}, Pavel Laskov, Giorgio Giacinto, and Fabio Roli.
	\newblock Evasion attacks against machine learning at test time.
	\newblock In {\em ECML-KDD}, pages 387--402. Springer, 2013.
	
	\bibitem{fawzi2015analysis}
	Alhussein Fawzi, Omar Fawzi, and Pascal Frossard.
	\newblock Analysis of classifiers' robustness to adversarial perturbations.
	\newblock {\em arXiv preprint arXiv:1502.02590}, 2015.
	
	\bibitem{goodfellow2014explaining}
	Ian~J Goodfellow, Jonathon Shlens, and Christian Szegedy.
	\newblock Explaining and harnessing adversarial examples.
	\newblock {\em arXiv preprint arXiv:1412.6572}, 2014.
	
	\bibitem{grosse2016adversarial}
	Kathrin Grosse, Nicolas Papernot, Praveen Manoharan, Michael Backes, and
	Patrick McDaniel.
	\newblock Adversarial perturbations against deep neural networks for malware
	classification.
	\newblock {\em arXiv preprint arXiv:1606.04435}, 2016.
	
	\bibitem{hu2017generating}
	Weiwei Hu and Ying Tan.
	\newblock Generating adversarial malware examples for black-box attacks based
	on gan.
	\newblock {\em arXiv preprint arXiv:1702.05983}, 2017.
	
	\bibitem{huang2017adversarial}
	Sandy Huang, Nicolas Papernot, Ian Goodfellow, Yan Duan, and Pieter Abbeel.
	\newblock Adversarial attacks on neural network policies.
	\newblock In {\em ICLR}, 2017.
	
	\bibitem{kurakin2016adversarial}
	Alexey Kurakin, Ian Goodfellow, and Samy Bengio.
	\newblock Adversarial examples in the physical world.
	\newblock In {\em ICLR}, 2017.
	
	\bibitem{lecun1998gradient}
	Yann LeCun, L{\'e}on Bottou, Yoshua Bengio, and Patrick Haffner.
	\newblock Gradient-based learning applied to document recognition.
	\newblock {\em Proceedings of the IEEE}, 86(11):2278--2324, 1998.
	
	\bibitem{lin2017tactics}
	Yen-Chen Lin, Zhang-Wei Hong, Yuan-Hong Liao, Meng-Li Shih, Ming-Yu Liu, and
	Min Sun.
	\newblock Tactics of adversarial attack on deep reinforcement learning agents.
	\newblock {\em arXiv preprint arXiv:1703.06748}, 2017.
	
	\bibitem{liu2016delving}
	Yanpei Liu, Xinyun Chen, Chang Liu, and Dawn Song.
	\newblock Delving into transferable adversarial examples and black-box attacks.
	\newblock {\em arXiv preprint arXiv:1611.02770}, 2016.
	
	\bibitem{moosavi2016deepfool}
	Seyed-Mohsen Moosavi-Dezfooli, Alhussein Fawzi, and Pascal Frossard.
	\newblock Deepfool: a simple and accurate method to fool deep neural networks.
	\newblock In {\em CVPR}, pages 2574--2582, 2016.
	
	\bibitem{papernot2016cleverhans}
	Nicolas Papernot, Ian Goodfellow, Ryan Sheatsley, Reuben Feinman, and Patrick
	McDaniel.
	\newblock cleverhans v1.0.0: an adversarial machine learning library.
	\newblock {\em arXiv preprint arXiv:1610.00768}, 2016.
	
	\bibitem{papernot2016transferability}
	Nicolas Papernot, Patrick McDaniel, and Ian Goodfellow.
	\newblock Transferability in machine learning: from phenomena to black-box
	attacks using adversarial samples.
	\newblock {\em arXiv preprint arXiv:1605.07277}, 2016.
	
	\bibitem{papernot2016practical}
	Nicolas Papernot, Patrick McDaniel, Ian Goodfellow, Somesh Jha, Z~Berkay Celik,
	and Ananthram Swami.
	\newblock Practical black-box attacks against deep learning systems using
	adversarial examples.
	\newblock {\em arXiv preprint arXiv:1602.02697}, 2016.
	
	\bibitem{papernot2016limitations}
	Nicolas Papernot, Patrick McDaniel, Somesh Jha, Matt Fredrikson, Z~Berkay
	Celik, and Ananthram Swami.
	\newblock The limitations of deep learning in adversarial settings.
	\newblock In {\em Security and Privacy (EuroS\&P), 2016 IEEE European Symposium
		on}, pages 372--387. IEEE, 2016.
	
	\bibitem{papernot2016towards}
	Nicolas Papernot, Patrick McDaniel, Arunesh Sinha, and Michael Wellman.
	\newblock Towards the science of security and privacy in machine learning.
	\newblock {\em arXiv preprint arXiv:1611.03814}, 2016.
	
	\bibitem{papernot2016distillation}
	Nicolas Papernot, Patrick McDaniel, Xi~Wu, Somesh Jha, and Ananthram Swami.
	\newblock Distillation as a defense to adversarial perturbations against deep
	neural networks.
	\newblock In {\em Security and Privacy (SP), 2016 IEEE Symposium on}, pages
	582--597. IEEE, 2016.
	
	\bibitem{laskov2014practical}
	Nedim {\v S}rndi{\'c} and Pavel Laskov.
	\newblock Practical evasion of a learning-based classifier: A case study.
	\newblock In {\em Security and Privacy (SP), 2014 IEEE Symposium on}, pages
	197--211. IEEE, 2014.
	
	\bibitem{szegedy2013intriguing}
	Christian Szegedy, Wojciech Zaremba, Ilya Sutskever, Joan Bruna, Dumitru Erhan,
	Ian Goodfellow, and Rob Fergus.
	\newblock Intriguing properties of neural networks.
	\newblock {\em arXiv preprint arXiv:1312.6199}, 2013.
	
	\bibitem{tramer2017ensemble}
	Florian Tram{\`e}r, Alexey Kurakin, Nicolas Papernot, Dan Boneh, and Patrick
	McDaniel.
	\newblock Ensemble adversarial training: Attacks and defenses.
	\newblock {\em arXiv preprint arXiv:1705.07204}, 2017.
	
	\bibitem{tramer2016stealing}
	Florian Tram{\`e}r, Fan Zhang, Ari Juels, Michael~K Reiter, and Thomas
	Ristenpart.
	\newblock Stealing machine learning models via prediction apis.
	\newblock In {\em Usenix Security}, 2016.
	
	\bibitem{WardeFarley16}
	David Warde-Farley and Ian Goodfellow.
	\newblock Adversarial perturbations of deep neural networks.
	\newblock In {\em Advanced Structured Prediction}. 2016.
	
	\bibitem{xu2016automatically}
	Weilin Xu, Yanjun Qi, and David Evans.
	\newblock Automatically evading classifiers.
	\newblock In {\em NDSS}, 2016.
	
\end{thebibliography}
\end{document}